\documentclass[letterpaper]{article} 
\usepackage{aaai2026}  
\usepackage{times}  
\usepackage{helvet}  
\usepackage{courier}  
\usepackage[hyphens]{url}  
\usepackage{graphicx} 
\usepackage{natbib}  
\usepackage{caption} 
\usepackage{algorithm}
\usepackage{algorithmic}

\usepackage{kotex}
\usepackage{amsmath}
\usepackage{amssymb}
\usepackage{multirow}
\usepackage{xcolor}
\usepackage{mathtools}
\usepackage{tabularx}

\usepackage{amsthm}
\newtheorem{theorem}{Theorem}
\newtheorem{proposition}{Proposition}
\newtheorem{lemma}{Lemma}
\newtheorem{definition}{Definition}

\usepackage{newfloat}
\usepackage{listings}
\DeclareCaptionStyle{ruled}{labelfont=normalfont,labelsep=colon,strut=off} 
\floatstyle{ruled}
\newfloat{listing}{tb}{lst}{}
\floatname{listing}{Listing}
\title{OID-PPO: Optimal Interior Design using Proximal Policy Optimization by Transforming Design Guidelines into Reward Functions}
\author {
    Chanyoung Yoon\textsuperscript{\rm 1},
    Sangbong Yoo\textsuperscript{\rm 2},
    Soobin Yim\textsuperscript{\rm 1},
    Chansoo Kim\textsuperscript{\rm 2},
    Yun Jang\textsuperscript{\rm 1}
}
\affiliations {
    \textsuperscript{\rm 1}Sejong University\\
    \textsuperscript{\rm 2}AI, Information and Reasoning (AI/R) Laboratory, Korea Institute of Science and Technology (KIST)\\
    vfgtr8746@gmail.com, usangbong@kist.re.kr, tn12qls@gmail.com, eau@kist.re.kr, jangy@sejong.edu
}

\usepackage{bibentry}

\begin{document}

\maketitle

\begin{abstract}
    Designing residential interiors strongly impacts occupant satisfaction but remains challenging due to unstructured spatial layouts, high computational demands, and reliance on expert knowledge. Existing methods based on optimization or deep learning are either computationally expensive or constrained by data scarcity. Reinforcement learning (RL) approaches often limit furniture placement to discrete positions and fail to incorporate design principles adequately. We propose OID-PPO, a novel RL framework for Optimal Interior Design using Proximal Policy Optimization, which integrates expert-defined functional and visual guidelines into a structured reward function. OID-PPO utilizes a diagonal Gaussian policy for continuous and flexible furniture placement, effectively exploring latent environmental dynamics under partial observability. Experiments conducted across diverse room shapes and furniture configurations demonstrate that OID-PPO significantly outperforms state-of-the-art methods in terms of layout quality and computational efficiency. Ablation studies further demonstrate the impact of structured guideline integration and reveal the distinct contributions of individual design constraints.
\end{abstract}

\section{Introduction}
\label{sec:intro}

Interior design plays a critical role in shaping comfort, functionality, and satisfaction in residential spaces, which accommodate essential daily activities such as cooking, cleaning, and sleeping~\cite{chen2015image2scene, wang2018deep, ritchie2019fast, stuart1982eval}. However, unlike structured environments such as hospitals or schools, residential spaces often lack standardized design frameworks, making objectives ambiguous and difficult to formalize~\cite{merrell2010computer, kan2017automated, liang2018knowledge}. Effective layout planning further relies on expert knowledge, creating barriers for both lay users and professionals~\cite{kan2017automated, liang2018knowledge}. The design process is inherently iterative and time-consuming, often requiring repeated interactions between designers and users~\cite{wu2019data, fisher2012example, merrell2011interactive}. While heuristic and rule-based guidelines have been proposed to aid this process~\cite{yu2011make}, systematically applying them in computational frameworks remains a significant challenge.

To address these challenges, prior research has explored the Optimal Interior Design (OID) problem using optimization, deep learning, and reinforcement learning (RL) techniques. Optimization-based approaches encode design guidelines as cost functions and solve OID by minimizing these objectives~\cite{yu2011make, merrell2011interactive}, but often suffer from high computational costs and large, non-convex search spaces~\cite{kan2018automatic}. Deep learning methods attempt to learn layout patterns directly from data~\cite{li2019grains, wang2019planit, hu2020graph2plan}, enabling layout generation with minimal user input, but require large-scale, high-quality datasets that are rarely available~\cite{csenbacslar2021rlss}. Reinforcement learning (RL) offers a promising alternative by learning design policies through interaction with the environment, thereby reducing the need for labeled data~\cite{csenbacslar2021rlss}. Recent RL-based approaches have been applied to OID~\cite{wang2019rlayout, ieq}, and advances in deep reinforcement learning (DRL) have improved training efficiency in complex environments~\cite{haisor, csenbacslar2021rlss}. However, existing DRL methods typically discretize furniture positions or rely on grid-based actions, which limits placement flexibility and fails to incorporate expert design principles in a structured manner.

In this paper, we propose OID-PPO (Optimizing Interior Design using Proximal Policy Optimization), a novel reinforcement learning framework for synthesizing high-quality interior layouts in continuous spatial domains. OID-PPO integrates two geometric encoders for representing the current and next furniture items, along with a convolutional encoder to capture room geometry and spatial semantics. Expert-defined design guidelines are embedded into a structured reward function, partitioned into functional and visual components. The agent employs a diagonal Gaussian policy to enable continuous and flexible furniture placement, while effectively exploring latent environmental dynamics under partial observability. We evaluate OID-PPO on diverse room geometries and furniture configurations, demonstrating that it consistently outperforms state-of-the-art optimization and DRL-based methods. Ablation studies further reveal the distinct contributions of individual reward components in achieving functionally valid, aesthetically pleasing, and spatially coherent layouts.
\section{Problem Definition}

The interior design process entails the sequential placement of furniture within a bounded room, subject to functional and visual constraints. This sequential decision-making structure aligns naturally with a Markov Decision Process (MDP), where each action corresponds to placing a furniture item and induces a transition in the spatial configuration. To model this formally and enable learning-based optimization, we formulate the single-room OID problem as a finite episodic MDP, defined by the tuple $\mathcal{M} = \langle S, A, P, R, \gamma \rangle$, where \( S \) is the state set, \( A \) the action set, \( P \) the transition function, \( R \) the design-based reward function, and \( \gamma \) the discount factor. The environment is represented as an axis-aligned quadrilateral room $E = [0, N] \times [0, M] \subset \mathbb{R}^2$, where $N$ and $M$ denote the room's width and height, respectively. The room boundary $\partial E$ contains at least one doorway, forming the door set $D \subset \partial E$.

At the start of each episode, the agent is provided with a finite set of furniture items $F$, sorted in descending order of footprint area to prioritize the placement of larger items. Each item $f \in F$ is associated with a canonical footprint $\Pi(f) \subset \mathbb{R}^2$, defined as an origin-centered geometric polygon. Furniture can be rotated using a discrete set of rotation operators $\mathbf{z} = { \mathbf{z}_0, \mathbf{z}_1, \mathbf{z}_2, \mathbf{z}_3 }$, where each $\mathbf{z}_k$ corresponds to a counterclockwise rotation by $90^\circ \times k$.

At each time step $t$, the placement action for furniture item $f_t$ consists of a two-dimensional position $\mathbf{x}_t = (x_f, y_f) \in \mathbb{R}^2$ and a rotation index $k_t \in {0, 1, 2, 3}$. The resulting footprint is computed via a rigid-body transformation $\mathcal{T}(\mathbf{x}, k;f) = \mathbf{z}_k \Pi(f) + \mathbf{x}$, where $\mathbf{z}_{k_t} \Pi(f_t)$ denotes the rotated footprint and $\mathbf{x}_t$ is the translation applied to position the furniture within the room. To proceed, we define the conditions that determine whether a placement action is valid.

\begin{definition}[Valid Placement]
    A placement action $a = (\mathbf{x}, k)$ for furniture item $f$ is considered valid if and only if the transformed footprint $\mathcal{T}(\mathbf{x}, k; f)$ satisfies both $\mathcal{T}(\mathbf{x}, k;f) \subset E$ and $\mathcal{T}(\mathbf{x},k;f)\cap \mathcal{O} = \emptyset$, where $E$ is the room boundary and $\mathcal{O}_t$ denotes the occupancy area at time $t$, defined as $\mathcal{O}_t = \bigcup_{i=0}^{t-1} \mathcal{T}(\mathbf{x}, k; f_i)$. Consequently, the feasible action set at time step $t$ is defined as $A_t = \{ (\mathbf{x}_t,k_t) \; | \; \mathcal{T}(\mathbf{x}_t, k_t;f_t) \subset E, \; \mathcal{T}(\mathbf{x}_t,k_t;f_t)\cap \mathcal{O}_t = \emptyset \}$.
\end{definition}

The state $s_t \in S$ at time step $t$ consists of three components: (i) the geometric descriptor $\mathbf{e}_t$ of the current furniture item $f_t$; (ii) the descriptor $\mathbf{e}_{t+1}$ of the next furniture item $f_{t+1}$; and (iii) a binary occupancy map $\mathbf{O}_t(x)$ that encodes the spatial footprint of all previously placed items. The occupancy map is defined as $\mathbf{O}_t(x) = \max_{f_{prev} \in F_{<t}} \mathcal{X}_{\mathcal{T}(\mathbf{x},k;f_{prev})} (x)$, where $\mathcal{X}_X(x) = 1_{x \in X}$. Given that the furniture set $F$ has finite cardinality, the episode horizon is naturally bounded. We formalize this observation with the following proposition:

\begin{proposition}[Finite horizon]
\label{prop:horizon}
    Every episode terminates after at most $|F|$ steps, thus the horizon $H$ satisfies $H \leq |F|$.
\end{proposition}
\begin{proof}
    At each time step, exactly one furniture item from the finite set $F$ is placed. Once all $|F|$ items have been placed, the episode terminates as no further actions remain.
\end{proof}

At each time step $t$, the action $a_t = (\mathbf{x}t, k_t) \in A \subset \mathbb{R}^2 \times {0, 1, 2, 3}$ specifies the placement position and rotation of furniture item $f_t$. If the action is valid ($a_t \in A_t$), the transition function $P$ deterministically moves the system to the next state $s_{t+1}$. Otherwise, if the action is invalid ($a_t \notin A_t$), the episode terminates immediately by transitioning to the terminal state $s_H$, and no further rewards are granted.

The reward function $R(s_t, a_t)$ aggregates multiple partial reward components, collectively denoted as $R_{\text{idg}}$, which capture adherence to established interior design guidelines (detailed in Section 3). Formally, the reward at each time step is defined as $R(s_t, a_t) = R_{\text{idg}}(s_t,a_t) \in [-1, 1]$. The agent's objective is to learn an optimal policy $\pi^*$ that maximizes the expected cumulative reward over the finite episode horizon:
\begin{equation*}
    \pi^* = \arg \max_\pi \mathbb{E}_\pi \left[ \sum_{t=0}^H \gamma^t R(s_t, a_t)\right]
\end{equation*}
where $\gamma \in (0, 1]$ is the discount factor. By integrating the MDP-based problem formulation, the explicit definition of placement validity, and the finite-horizon condition established above, we lay a rigorous theoretical foundation for developing an RL framework tailored to solving the OID problem, as detailed in the subsequent sections.
\section{Interior Design Guidelines}

Interior design has traditionally relied on expert intuition rather than formal quantitative criteria. To enable computational reasoning and automation, prior work has translated these intuitive practices into explicit and codified design guidelines~\cite{yu2011make, merrell2011interactive}. Building upon this foundation, we formalize both functional and visual criteria as quantifiable reward functions, which are integrated into a unified guideline-based reward to support learning-based layout optimization.

\noindent \textbf{Pairwise Relationship.}
Certain furniture pairs form functionally cohesive units, such as a desk and chair, where misalignment or excessive separation hinders usability and spatial coherence. To promote proper pairing, we introduce a reward function composed of two components: a distance-based kernel $K_{\text{dist}}$ and a directional alignment kernel $K_{\text{dir}}$:
{\footnotesize
\begin{equation*}
    K_{\text{dist}}(p,c) = 1 + \cos \left( \frac{\pi d_{pc}}{d_\triangle} \right), \; K_{\text{dir}}(p,c) = \frac{1+\alpha_{pc} \langle \textbf{n}_p, \textbf{n}_c \rangle}{2}
\end{equation*}
}
Here, $d_{pc}$ denotes the Euclidean distance between the centers of the parent $p$ and child $c$, while $d_\triangle = \sqrt{N^2 + M^2}$ represents the diagonal length of the room. The unit vectors $\mathbf{n}_p$ and $\mathbf{n}_c$ indicate the front-facing directions of the respective furniture items. The inner product $\langle \mathbf{n}_p, \mathbf{n}_c \rangle$ quantifies their directional alignment. The parameter $\alpha_{pc} \in \{-1, 1\}$ specifies the preferred orientation: $\alpha_{pc} = -1$ encourages face-to-face alignment, while $\alpha_{pc} = +1$ promotes parallel alignment. Based on the distance and directional kernels, the pairwise reward function $R_{\text{pair}}$ is defined as:
\begin{equation*}
    R_{\text{pair}} = \frac{1}{|\mathcal{P}|} \sum_{(p,c) \in \mathcal{P}} K_{\text{dist}}(p,c) \cdot K_{\text{dir}}(p,c)
\end{equation*}
where $\mathcal{P}$ denotes the set of predefined parent–child furniture pairs. Maximizing $R_{\text{pair}}$ encourages the agent to place functionally related items nearby with appropriate orientation.

\noindent \textbf{Accessibility.}
Accessibility quantifies how easily users can approach and interact with furniture. We formally define the associated access constraints as follows.

\begin{definition}[Access Area and Violation Area]
    For each furniture item $f$, let $\mathbf{n}_f$ denote the front-facing direction, and $\mathbf{n}_f^\bot$ its orthogonal vector. Define the set of required access directions as $\mathcal{D}_f = \{\pm \mathbf{n}_f, \pm \mathbf{n}_f^\bot\}$, with each direction $\bar{\mathbf{d}}$ associated with a minimum clearance offset $\Omega_{f,\bar{\mathbf{d}}}$. Let $U_f = \bigcup_{\bar{\mathbf{d}} \in \mathcal{D}_f} U_{f,\bar{\mathbf{d}}}$. The required access area is defined via the Minkowski sum as $U_{f,\bar{\mathbf{d}}} = \Pi(f) \;\oplus \; \{ \Delta \bar{\mathbf{d}} \; | \; 0 < \Delta \leq \Omega_{f,\bar{\mathbf{d}}} \}$. Let $F^{\neg \mathcal{P}}$ denote the predefined set of furniture not functionally paired with $f$. The accessibility violation area $\nu(f)$ is the portion of the required access area $U_f$ that intersects with the footprint of non-paired items $\nu(f) = U_f \cap \bigcup_{q \in F^{\neg \mathcal{P}}(f)} \Pi(q)$ where $F^{\neg \mathcal{P}}(f) = \{ q \in F \; | \; q \neq f, \;(f, q) \notin \mathcal{P} \}$.
\end{definition}

The accessibility reward $R_a$ is defined as the average proportion of obstructed access areas across all furniture:
\begin{equation*}
    R_a = 1- \frac{2}{|F|} \sum_{f \in F} \frac{|\nu(f)|}{|U_f|}
\end{equation*}
Here, $R_a = 1$ indicates that all furniture items are fully accessible, whereas $R_a = -1$ corresponds to complete inaccessibility. The OID-PPO agent is thus incentivized to maximize $R_a$ by preserving necessary clearance around furniture and ensuring usability.

\noindent \textbf{Visibility.}
To promote usability, we penalize layouts where the front of a furniture item directly faces a wall. Let $\mathbf{n}_{w(f)}$ be the normal vector of the wall closest to furniture $f$. The visibility reward is defined as:
\begin{equation*}
    R_v = -\frac{1}{|F|} \sum_{f \in F} \langle \mathbf{n}_f, \mathbf{n}_{w(f)}\rangle
\end{equation*}
A dot product value of $+1$ indicates that the front of the furniture is directly facing the wall, resulting in the maximum penalty. Conversely, a value of $-1$ signifies that the furniture front is oriented away from the wall, yielding the maximum reward. By maximizing $R_v$, the agent learns to orient furniture toward open areas, thereby enhancing both usability and visual comfort.

\noindent \textbf{Pathway Connection.}
A feasible interior layout requires that each furniture item be reachable from at least one doorway without obstruction and that adequate spatial buffers are maintained near doorways. For each furniture item $f$, we calculate its proximity to the nearest doorway using the Euclidean distance $d_{\text{door}}$, and assess its accessibility by computing the shortest unobstructed path from the doorway using A* search on a discretized occupancy grid. Reachability is formally defined as follows:

\begin{definition}[Reachability Distance]
    For each furniture item $f$, the reachability distance $\rho(f)$ is defined as the length of the shortest unobstructed path from any doorway $D$ to the furniture center $\mathbf{x}$, computed using A* search on a discretized occupancy grid. If no such path exists due to obstructions, then $\rho(f) = \infty$.
\end{definition}

Based on this measure, the pathway connection reward $R_{\text{path}}$ integrates both reachability and doorway proximity:
\begin{equation*}
    R_{\text{path}} = 1-\frac{2}{|F|} \sum_{f \in F} \left[ (1-I_f)+ e^{-\kappa_f}I_f\right]
\end{equation*}
where $\kappa_f = (d_{\text{door}}/d_\triangle)^2$ and $I_f = \mathcal{X}_{\rho(f) < \infty}$ is an indicator equal to 1 if $f$ is reachable and 0 otherwise. Maximizing $R_{\text{path}}$ encourages the agent to maintain clear paths and unobstructed access to furniture from doorways. We further establish a theoretical guarantee that reachable layouts consistently receive higher rewards:

\begin{proposition}[Reachable Layouts Receive Higher Pathway Reward]
\label{prop:pathway}
    Let $\mathcal{L}_{\text{conn}}$ be a layout where every furniture is reachable~(i.e., $\rho(f) < \infty, \; \forall f \in F$), and let $\mathcal{L}_{\text{disc}}$ be an otherwise identical layout where at least one furniture $f' \in F$ is unreachable~(i.e., $\rho(f') = \infty$). Then $R_{\text{path}}(\mathcal{L}_{\text{conn}}) > R_{\text{path}}(\mathcal{L}_{\text{disc}})$.
\end{proposition}

\begin{proof}
    For reachable $f$, the summation term is $T_f^{\text{conn}} = e^{-\kappa_f} < 1$. For unreachable $f'$, the term becomes $T_{f'}^{\text{disc}} = 1$. Since $T_f^{\text{conn}} < T_{f'}^{\text{disc}}$, the total sum in $\mathcal{L}_{\text{conn}}$ is smaller, leading to a higher $R_{\text{path}}$.
\end{proof}

This proposition rigorously supports the guideline that layouts maintaining full accessibility and clear pathways achieve higher rewards, thereby aligning the agent's learned policy with practical interior design objectives.

\noindent \textbf{Visual Balance.}
Human observers naturally assess the visual center of mass within a room and the extent to which this mass is evenly distributed~\cite{zejnilovic2023perception}. Layouts that feature excessive clustering in a corner or a single, large, isolated object often appear unbalanced and inefficient in terms of space utilization. Let the total footprint area be denoted as $\Pi(F) = \sum_{f \in F} \Pi(f)$, and define the area-weighted center of mass as $\bar{\mathbf{x}}_F = \frac{1}{\Pi(F)} \sum_{f} \Pi(f) \mathbf{x}_f$. To measure spatial dispersion around this centroid, we define the spatial variance tensor:

\begin{definition}[Spatial Variance Tensor]
    The spatial variance tensor $\Sigma_F$ of the furniture distribution is computed as $\Sigma_F = \frac{1}{\Pi(F)} \sum_f \Pi(f)(\mathbf{x}_f-\bar{\mathbf{x}}_F)(\mathbf{x}_f-\bar{\mathbf{x}}_F)^T$.
\end{definition}

Let $\mathbf{o}$ denote the geometric center of the room. The reference spatial variance is defined as $\varkappa_E^2 = (N^2 + M^2)/12$, which corresponds to the spatial variance of a uniform mass distribution over a quadrilateral room. Based on this, the visual balance reward $R_b$ is given by:
\begin{equation*}
    R_b = \exp(-\frac{||\bar{\mathbf{x}}_F-\mathbf{o}||_2^2}{d_\triangle^2}) + \exp(-\frac{||\Sigma_F-\varkappa_E^2I||_F^2}{\varkappa_E^4}) - 1
\end{equation*}
Here, $||\cdot||_F$ denotes the Frobenius norm and $I$ is the identity matrix. The first exponential penalizes displacement of the mass center from the room center, while the second penalizes deviations of the variance tensor from the ideal reference. Thus, $R_b = +1$ is achieved when the layout is perfectly centered and evenly distributed.

\noindent \textbf{Alignment.}
The alignment constraint encourages furniture to align its long axis with the nearest wall boundary $\partial D$, either parallel or perpendicular. For each furniture $f \in F \setminus \mathcal{C}$, we define its unit orientation vector $\mathbf{u}_f$ and the tangent direction of the closest wall as $\tau_{w(f)} = \mathbf{n}_{w(f)}^\bot / \|\mathbf{n}_{w(f)}^\bot\|$, and compute the angular deviation as $\vartheta_f = \arccos\left(|\mathbf{u}_f \cdot \tau_{w(f)}|\right)$.

To capture proximity, we define the normalized wall clearance as $\omega_f = d_f / \ell_f$, where $d_f$ is the shortest back or side distance from the furniture to the wall along $\tau_{w(f)}$, and $\ell_f$ is the length of the long axis. By combining angular alignment and proximity, the alignment reward $R_{\text{al}}$ is defined as:
\begin{equation*}
    R_{\text{al}} = \frac{\sum_{f \in F \setminus \mathcal{C}} \Pi(f) \cos^2(2\vartheta_f)(1-\tanh^2 \omega_f)}{\sum_{f \in F \setminus \mathcal{C}} \Pi(f)}
\end{equation*}
The squared-cosine term reaches its maximum when furniture is aligned at $0^\circ$ or $90^\circ$ to the wall, while the hyperbolic term penalizes excessive distance. An alignment reward close to $+1$ indicates well-aligned and tightly placed furniture, whereas values near $-1$ signal misaligned or poorly positioned items. Hence, the agent is guided to place furniture with consistent orientation and wall alignment.

\noindent \textbf{Guideline Reward Function}
Since all partial rewards---$R_{\text{pair}}, R_a, R_v, R_{\text{path}}, R_b, R_{\text{al}}$---are normalized within the interval $[-1, 1]$, their arithmetic mean defines a composite reward function on a consistent scale:
\begin{equation*}
    R_{\text{idg}} = \frac{1}{6}(R_{\text{pair}} + R_a + R_v + R_{\text{path}} + R_b + R_{\text{al}}) \in [-1, 1]
\end{equation*}
To ensure consistency, we present the following lemma.

\begin{lemma}[Reward Normalization]
\label{lemma:reward}
    Each guideline reward is bounded within $[-1, 1]$. Therefore, their arithmetic mean $R_{\text{idg}}$ is also bounded within $[-1, 1]$.
\end{lemma}

\begin{proof}
    Each partial reward function is explicitly designed and normalized to lie within the range $[-1, 1]$. As such, their arithmetic mean remains within the same bounds.
\end{proof}

In addition, any infeasible action is penalized with a terminal reward $\varphi$ and terminates the episode immediately without further reward. Thus, the agent learns to maximize the overall reward by balancing spatial feasibility, functional usability, and visual quality.
\section{Model}

\begin{figure}[tp]
    \centering
    \includegraphics[width=\linewidth]{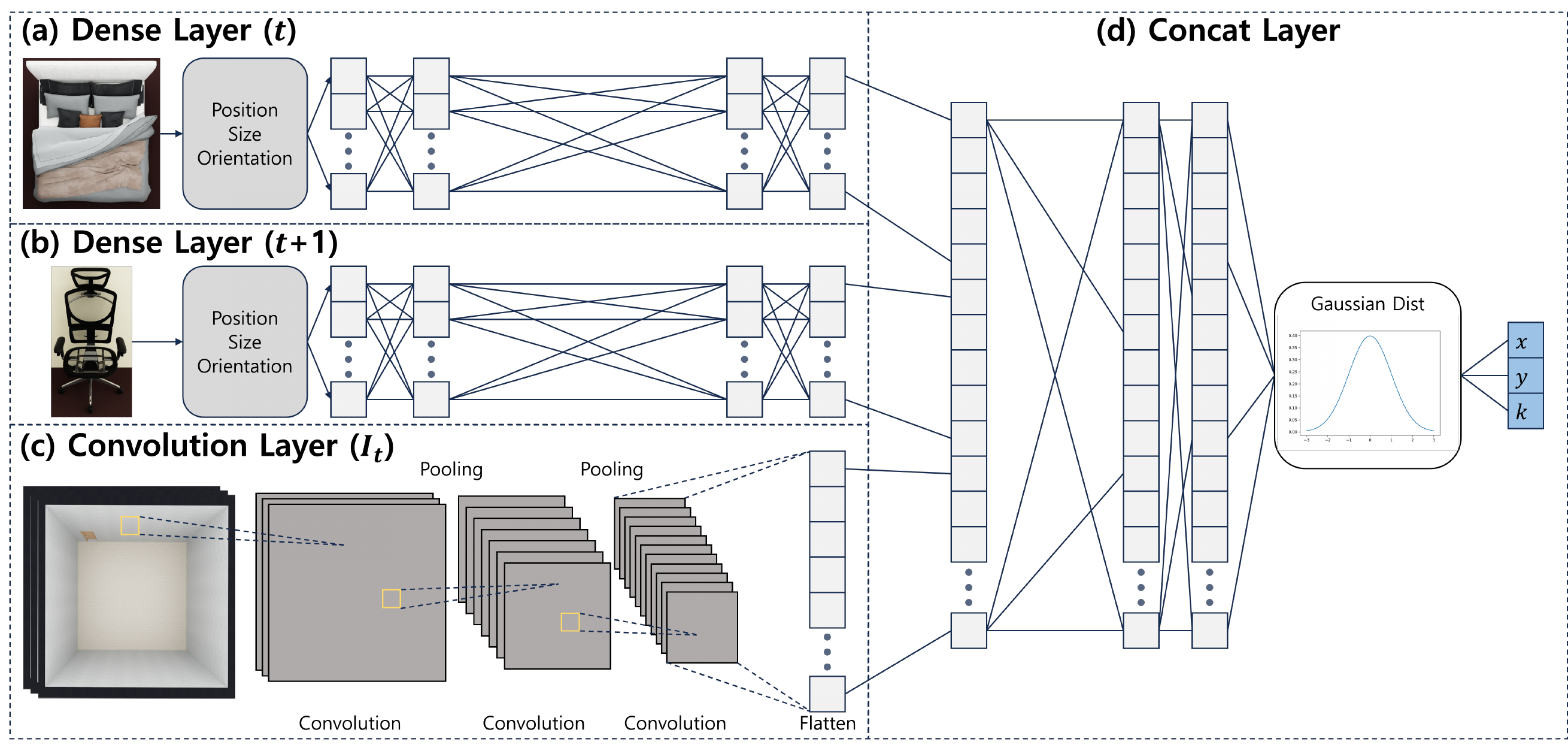}
    \caption{OID-PPO architecture: (a)–(b) embed current and next furniture via dense layers; (c) encodes the occupancy map via a CNN; (d) concatenates features and outputs a diagonal Gaussian policy for continuous action sampling. This design supports exploration under partial observability.}
    \label{fig:model}
\end{figure}

The OID-PPO framework adopts an actor-critic architecture, where the actor learns the placement policy and the critic estimates the value function. As shown in Figure~\ref{fig:model}, both the current and subsequent furniture descriptors are encoded through identical $L$-layer MLPs with a GELU activation function:
\begin{equation*}
    \psi_{\text{obj}}(\mathbf{e}) = \sigma(W_L, \cdots \sigma(W_2 \sigma(W_1\mathbf{e} + b_1) + \cdots ) +b_L)
\end{equation*}
The binary occupancy map $\mathbf{O}_t$ is embedded via a CNN to produce $\psi_\mathbf{O}$. The feature vectors from the current furniture $\psi_t$, the next furniture $\psi_{t+1}$, and the occupancy map $\psi_\mathbf{O}$ are concatenated to form the joint representation as $\mathbf{h}_t=\text{concat} \left[ \psi_t, \psi_{t+1}, \psi_\mathbf{O} \right]$. A fully connected actor head then outputs the parameters of a diagonal Gaussian policy for action sampling as $\mu_t = W_\mu \mathbf{h}_t + b_\mu$ and $\log \sigma_t = W_\sigma \mathbf{h}_t + b_\sigma$. Actions are sampled from a diagonal Gaussian policy as $a_t = \mu_t + \sigma_t \odot z$, where $z \sim \mathcal{N}(0, I)$. The critic head maps the shared embedding $\phi_t$ to a scalar value estimate using a linear projection: $V_\phi(s_t)=W_v\phi_t+b_v$.

The OID-PPO agent is trained using PPO, which stabilizes updates through clipped surrogate objectives and Generalized Advantage Estimation (GAE)~\cite{gae}. The temporal-difference error is computed as $\delta_t = R_t + \gamma V_\phi(s_{t+1}) - V_\phi(s_t)$, and the advantage estimate is $\hat{A}_t = \sum_{l=0}^{T-1} (\gamma \lambda)^l \delta_{t+l}$. Letting $r_t(\theta) = \pi_\theta(a_t|s_t) / \pi_{\theta_{\text{old}}}(a_t|s_t)$ denote the probability ratio, the clipped PPO objective over minibatch $\mathcal{B}$ is defined as:
{\footnotesize
\begin{equation*}
    \mathcal{L}^{\text{clip}}(\theta) = \frac{1}{|\mathcal{B}|} \sum_{t \in \mathcal{B}} \min(r_t(\theta)\hat{A_t}, \text{clip}(r_t(\theta), 1-\epsilon, 1+\epsilon)\hat{A_t})
\end{equation*}
}
Here, $\epsilon$ denotes the clipping threshold. The total loss combines policy, value, and entropy terms as follows:
\begin{equation*}
    \mathcal{L}(\theta, \phi) = -\mathcal{L}^{\text{clip}}(\theta) + c_v\mathcal{L}^{VF}(\phi)-c_e\mathcal{H}(\pi_\theta)
\end{equation*}
where the value function loss is defined as $\mathcal{L}^{VF}(\phi) = \frac{1}{|\mathcal{B}|} \sum_{t \in \mathcal{B}} (V_\phi(s_t) - R_t^\lambda)^2$, and $\mathcal{H}(\pi_\theta)$ denotes the differentiable entropy of the policy. Based on this loss, the gradients for the actor and critic are computed as:
\begin{equation*}
g_\theta = \nabla_\theta (-L^{\text{clip}}(\theta)-c_e \mathcal{H}(\pi_\theta)), \; g_\phi = c_v \nabla_\phi L^{VF}(\phi)    
\end{equation*}
Parameter updates leverage the Adam optimizer with bias correction:
\begin{equation*}
\begin{array}{c}
    m \leftarrow \beta_1 m + (1-\beta_1)g, \; v \leftarrow \beta_2v + (1-\beta_2)g^2\\
    \theta \leftarrow \theta - \eta_\theta \frac{m_\theta}{\sqrt{v_\theta} + \epsilon_{\text{adam}}}, \quad \phi \leftarrow \phi - \eta_\phi \frac{m_\phi}{\sqrt{v_\phi} + \epsilon_{\text{adam}}}
\end{array}
\end{equation*}
The diagonal Gaussian policy and the PPO optimization procedure exhibit key theoretical properties, which are formally described in the following propositions.

\begin{proposition}[Monotonic Policy Improvement]
\label{prop:monotonic}
    The PPO clipped surrogate objective guarantees monotonic improvement in policy performance, provided that the step sizes are sufficiently small.
\end{proposition}

\begin{proposition}[Guideline-Aware Exploration]
\label{prop:exploration}
    Because the diagonal Gaussian policy maintains nonzero variance, the OID-PPO agent guarantees a strictly positive probability of exploring all valid furniture placements.
\end{proposition}

These propositions ensure both stable policy optimization and sufficient exploration during training. Finally, we integrate the conditions established in the preceding sections to provide theoretical convergence guarantees for the OID-PPO model.

\begin{theorem}[Convergence of OID-PPO]
    Consider the finite episodic MDP formulation of the OID problem, defined as $\mathcal{M} = \langle S,A,P,R,\gamma\rangle$, where $\gamma \in (0,1]$ and the horizon of episodes $H=|F|\in\mathbb{N}$ is finite. 
    Let $\{\pi_{\theta_k}\}_{k\ge 0}$ denote the parameterized policy sequence by $\theta\in\mathbb{R}^n$, which is updated through PPO using the clipped surrogate objective $L_{\text{clip}}(\theta)$ and GAE. The PPO gradient update is explicitly given by $g_{\theta} = \nabla_{\theta}\left(-L_{\text{clip}}(\theta) - c_e H(\pi_{\theta})\right)$ without additional noise terms and the parameters are updated using the Adam optimizer with bias correction as previously described.

    Assume each episode terminates within at most $H=|F|$ steps (Proposition~\ref{prop:horizon}); the composite reward $R_{\text{idg}}(s_t,a_t)$ is bounded within $[-1, 1]$, with a finite penalty $\phi>-\infty$ for invalid actions (Lemma~\ref{lemma:reward}); the diagonal Gaussian policy maintains strictly positive variance at all times for persistent exploration (Proposition~\ref{prop:exploration}); the PPO clipped surrogate $L_{\text{clip}}(\theta)$ ensures monotonic policy improvement with bounded bias, satisfying $|J(\theta)-L_{\text{clip}}(\theta)| \leq O(\epsilon)$ (Proposition~\ref{prop:monotonic}); and $L_{\text{clip}}(\theta)$ is continuously differentiable with Lipschitz-continuous gradients on compact parameter sets.

    Then, the policy parameters $\theta_k$ almost surely converge to a locally optimal parameter $\theta^*$, such that the true return satisfies $J(\theta_k)\xrightarrow{\text{a.s.}}J(\theta^*)$, where $\theta^*\in\arg\max_\theta L_{\text{clip}}(\theta)$. Furthermore, the expected return sequence $J(\theta_k)$ is monotonically non-decreasing up to an $O(\epsilon)$ clipping bias.
\end{theorem}

\begin{proof}
    See Appendix~B. This convergence result provides strong theoretical support for the stability and reliability of our RL-based interior design framework.
\end{proof}
\section{Experiment}

\subsection{Environment Setting}
We implement OID-PPO in PyTorch and deploy it within the OpenAI Gym environment. The model encodes the binary occupancy map $\mathbf{O}_t$ through convolutional layers and embeds the current and next furniture descriptors, $\mathbf{e}_t$ and $\mathbf{e}_{t+1}$, via dense layers. These three outputs are concatenated and passed through fully connected layers to predict the mean $\mu$ and standard deviation $\sigma$ of a diagonal Gaussian action distribution. For training, we set $\gamma=0.99$, GAE $\lambda=0.95$, learning rates $\eta_a=10^{-4}$ and $\eta_c=10^{-3}$ for the actor and critic, clipping ratio $\epsilon=0.2$, penalty $\varphi=-10$, and epochs to 1000. To evaluate generalization, each episode randomly samples one of four room shapes---square, rectangular, L-shape, or U-shape---and a furniture count from $F_n \in \{4, 6, 8\}$.

\begin{table*}[h]
\centering
\caption{Performance comparison of various interior-layout models.
\textcolor{black}{
P-Loss is the policy loss, and V-Loss is the value loss.}}
\label{table:oidppo}
\renewcommand{\arraystretch}{1.2}\setlength{\tabcolsep}{1mm}\small
\begin{tabular}{cc||cccc||cccc||cccc}
\hline\hline\hline
\multicolumn{2}{l||}{\multirow{2}{*}{}} &
  \multicolumn{4}{c||}{$F_n=4$} &
  \multicolumn{4}{c||}{$F_n=6$} &
  \multicolumn{4}{c}{$F_n=8$}\\ \cline{3-14}
\multicolumn{2}{l||}{} &
  \multicolumn{1}{c|}{Time (s)} & \multicolumn{1}{c|}{P-Loss} &
  \multicolumn{1}{c|}{V-Loss} & Reward &
  \multicolumn{1}{c|}{Time (s)} & \multicolumn{1}{c|}{P-Loss} &
  \multicolumn{1}{c|}{V-Loss} & Reward &
  \multicolumn{1}{c|}{Time (s)} & \multicolumn{1}{c|}{P-Loss} &
  \multicolumn{1}{c|}{V-Loss} & Reward \\ \hline\hline
\multicolumn{1}{l|}{\multirow{4}{*}{\rotatebox[origin=c]{90}{\textbf{MH}}}}&
 Square    & \multicolumn{1}{c|}{70.28} & \multicolumn{1}{c|}{–} & \multicolumn{1}{c|}{–} & 0.281 &
              \multicolumn{1}{c|}{94.92} & \multicolumn{1}{c|}{–} & \multicolumn{1}{c|}{–} & 0.214 &
              \multicolumn{1}{c|}{123.4} & \multicolumn{1}{c|}{–} & \multicolumn{1}{c|}{–} & 0.126\\ \cline{2-14}
\multicolumn{1}{l|}{}&
 Rectangle & \multicolumn{1}{c|}{63.79} & \multicolumn{1}{c|}{–} & \multicolumn{1}{c|}{–} & 0.265 &
              \multicolumn{1}{c|}{102.2} & \multicolumn{1}{c|}{–} & \multicolumn{1}{c|}{–} & 0.198 &
              \multicolumn{1}{c|}{131.2} & \multicolumn{1}{c|}{–} & \multicolumn{1}{c|}{–} & 0.143\\ \cline{2-14}
\multicolumn{1}{l|}{}&
 L-Shape   & \multicolumn{1}{c|}{83.12} & \multicolumn{1}{c|}{–} & \multicolumn{1}{c|}{–} & 0.237 &
              \multicolumn{1}{c|}{115.7} & \multicolumn{1}{c|}{–} & \multicolumn{1}{c|}{–} & 0.165 &
              \multicolumn{1}{c|}{138.4} & \multicolumn{1}{c|}{–} & \multicolumn{1}{c|}{–} & 0.112\\ \cline{2-14}
\multicolumn{1}{l|}{}&
 U-Shape   & \multicolumn{1}{c|}{77.64} & \multicolumn{1}{c|}{–} & \multicolumn{1}{c|}{–} & 0.221 &
              \multicolumn{1}{c|}{109.6} & \multicolumn{1}{c|}{–} & \multicolumn{1}{c|}{–} & 0.135 &
              \multicolumn{1}{c|}{145.6} & \multicolumn{1}{c|}{–} & \multicolumn{1}{c|}{–} & 0.105\\ \hline\hline
\multicolumn{1}{l|}{\multirow{4}{*}{\rotatebox[origin=c]{90}{\textbf{MOPSO}}}}&
 Square    & \multicolumn{1}{c|}{96.82} & \multicolumn{1}{c|}{–} & \multicolumn{1}{c|}{–} & 0.338 &
              \multicolumn{1}{c|}{150.8} & \multicolumn{1}{c|}{–} & \multicolumn{1}{c|}{–} & 0.259 &
              \multicolumn{1}{c|}{184.7} & \multicolumn{1}{c|}{–} & \multicolumn{1}{c|}{–} & 0.207\\ \cline{2-14}
\multicolumn{1}{l|}{}&
 Rectangle & \multicolumn{1}{c|}{92.47} & \multicolumn{1}{c|}{–} & \multicolumn{1}{c|}{–} & 0.352 &
              \multicolumn{1}{c|}{144.4} & \multicolumn{1}{c|}{–} & \multicolumn{1}{c|}{–} & 0.271 &
              \multicolumn{1}{c|}{191.4} & \multicolumn{1}{c|}{–} & \multicolumn{1}{c|}{–} & 0.192\\ \cline{2-14}
\multicolumn{1}{l|}{}&
 L-Shape   & \multicolumn{1}{c|}{109.7} & \multicolumn{1}{c|}{–} & \multicolumn{1}{c|}{–} & 0.305 &
              \multicolumn{1}{c|}{159.3} & \multicolumn{1}{c|}{–} & \multicolumn{1}{c|}{–} & 0.214 &
              \multicolumn{1}{c|}{208.0} & \multicolumn{1}{c|}{–} & \multicolumn{1}{c|}{–} & 0.162\\ \cline{2-14}
\multicolumn{1}{l|}{}&
 U-Shape   & \multicolumn{1}{c|}{103.5} & \multicolumn{1}{c|}{–} & \multicolumn{1}{c|}{–} & 0.281 &
              \multicolumn{1}{c|}{166.4} & \multicolumn{1}{c|}{–} & \multicolumn{1}{c|}{–} & 0.223 &
              \multicolumn{1}{c|}{199.8} & \multicolumn{1}{c|}{–} & \multicolumn{1}{c|}{–} & 0.157\\ \hline\hline
\multicolumn{1}{l|}{\multirow{4}{*}{\rotatebox[origin=c]{90}{\textbf{DDPG}}}}&
 Square    & \multicolumn{1}{c|}{1.247} & \multicolumn{1}{c|}{0.065} & \multicolumn{1}{c|}{0.233} & 0.803 &
              \multicolumn{1}{c|}{\textbf{1.698}} & \multicolumn{1}{c|}{0.087} & \multicolumn{1}{c|}{0.289} & 0.783 &
              \multicolumn{1}{c|}{\textbf{2.318}} & \multicolumn{1}{c|}{0.109} & \multicolumn{1}{c|}{0.337} & 0.639\\ \cline{2-14}
\multicolumn{1}{l|}{}&
 Rectangle & \multicolumn{1}{c|}{1.268} & \multicolumn{1}{c|}{0.068} & \multicolumn{1}{c|}{0.238} & 0.812 &
              \multicolumn{1}{c|}{\textbf{1.792}} & \multicolumn{1}{c|}{0.089} & \multicolumn{1}{c|}{0.296} & 0.792 &
              \multicolumn{1}{c|}{\textbf{2.389}} & \multicolumn{1}{c|}{0.111} & \multicolumn{1}{c|}{0.343} & 0.648\\ \cline{2-14}
\multicolumn{1}{l|}{}&
 L-Shape   & \multicolumn{1}{c|}{\textbf{1.481}} & \multicolumn{1}{c|}{0.072} & \multicolumn{1}{c|}{0.250} & 0.769 &
              \multicolumn{1}{c|}{\textbf{2.017}} & \multicolumn{1}{c|}{0.093} & \multicolumn{1}{c|}{0.304} & 0.755 &
              \multicolumn{1}{c|}{3.036} & \multicolumn{1}{c|}{0.114} & \multicolumn{1}{c|}{0.353} & 0.618\\ \cline{2-14}
\multicolumn{1}{l|}{}&
 U-Shape   & \multicolumn{1}{c|}{\textbf{1.608}} & \multicolumn{1}{c|}{0.075} & \multicolumn{1}{c|}{0.259} & 0.758 &
              \multicolumn{1}{c|}{\textbf{2.158}} & \multicolumn{1}{c|}{0.096} & \multicolumn{1}{c|}{0.315} & 0.745 &
              \multicolumn{1}{c|}{\textbf{2.829}} & \multicolumn{1}{c|}{0.117} & \multicolumn{1}{c|}{0.364} & 0.609\\ \hline\hline
\multicolumn{1}{l|}{\multirow{4}{*}{\rotatebox[origin=c]{90}{\textbf{TD3}}}}&
 Square    & \multicolumn{1}{c|}{\textbf{1.181}} & \multicolumn{1}{c|}{0.061} & \multicolumn{1}{c|}{0.206} & 0.823 &
              \multicolumn{1}{c|}{1.923} & \multicolumn{1}{c|}{0.082} & \multicolumn{1}{c|}{0.256} & 0.802 &
              \multicolumn{1}{c|}{2.887} & \multicolumn{1}{c|}{0.103} & \multicolumn{1}{c|}{0.299} & 0.671\\ \cline{2-14}
\multicolumn{1}{l|}{}&
 Rectangle & \multicolumn{1}{c|}{\textbf{1.153}} & \multicolumn{1}{c|}{0.063} & \multicolumn{1}{c|}{0.210} & 0.803 &
              \multicolumn{1}{c|}{1.862} & \multicolumn{1}{c|}{0.084} & \multicolumn{1}{c|}{0.263} & 0.787 &
              \multicolumn{1}{c|}{2.964} & \multicolumn{1}{c|}{0.105} & \multicolumn{1}{c|}{0.306} & 0.661\\ \cline{2-14}
\multicolumn{1}{l|}{}&
 L-Shape   & \multicolumn{1}{c|}{1.603} & \multicolumn{1}{c|}{0.067} & \multicolumn{1}{c|}{0.221} & 0.792 &
              \multicolumn{1}{c|}{2.427} & \multicolumn{1}{c|}{0.088} & \multicolumn{1}{c|}{0.273} & 0.770 &
              \multicolumn{1}{c|}{\textbf{3.012}} & \multicolumn{1}{c|}{0.110} & \multicolumn{1}{c|}{0.318} & 0.641\\ \cline{2-14}
\multicolumn{1}{l|}{}&
 U-Shape   & \multicolumn{1}{c|}{1.714} & \multicolumn{1}{c|}{0.070} & \multicolumn{1}{c|}{0.231} & 0.775 &
              \multicolumn{1}{c|}{2.509} & \multicolumn{1}{c|}{0.091} & \multicolumn{1}{c|}{0.283} & 0.761 &
              \multicolumn{1}{c|}{3.128} & \multicolumn{1}{c|}{0.113} & \multicolumn{1}{c|}{0.329} & 0.632\\ \hline\hline
\multicolumn{1}{l|}{\multirow{4}{*}{\rotatebox[origin=c]{90}{\textbf{SAC}}}}&
 Square    & \multicolumn{1}{c|}{2.547} & \multicolumn{1}{c|}{0.045} & \multicolumn{1}{c|}{0.134} & 0.903 &
              \multicolumn{1}{c|}{3.827} & \multicolumn{1}{c|}{0.055} & \multicolumn{1}{c|}{0.186} & 0.891 &
              \multicolumn{1}{c|}{5.507} & \multicolumn{1}{c|}{0.066} & \multicolumn{1}{c|}{0.226} & 0.831\\ \cline{2-14}
\multicolumn{1}{l|}{}&
 Rectangle & \multicolumn{1}{c|}{2.758} & \multicolumn{1}{c|}{0.042} & \multicolumn{1}{c|}{0.130} & 0.908 &
              \multicolumn{1}{c|}{4.052} & \multicolumn{1}{c|}{0.053} & \multicolumn{1}{c|}{0.181} & 0.894 &
              \multicolumn{1}{c|}{5.785} & \multicolumn{1}{c|}{0.065} & \multicolumn{1}{c|}{0.224} & 0.834\\ \cline{2-14}
\multicolumn{1}{l|}{}&
 L-Shape   & \multicolumn{1}{c|}{3.063} & \multicolumn{1}{c|}{0.052} & \multicolumn{1}{c|}{0.152} & 0.865 &
              \multicolumn{1}{c|}{4.358} & \multicolumn{1}{c|}{0.063} & \multicolumn{1}{c|}{0.203} & 0.852 &
              \multicolumn{1}{c|}{6.030} & \multicolumn{1}{c|}{0.074} & \multicolumn{1}{c|}{0.244} & 0.785\\ \cline{2-14}
\multicolumn{1}{l|}{}&
 U-Shape   & \multicolumn{1}{c|}{3.268} & \multicolumn{1}{c|}{0.055} & \multicolumn{1}{c|}{0.158} & 0.852 &
              \multicolumn{1}{c|}{4.547} & \multicolumn{1}{c|}{0.066} & \multicolumn{1}{c|}{0.209} & 0.842 &
              \multicolumn{1}{c|}{6.215} & \multicolumn{1}{c|}{0.076} & \multicolumn{1}{c|}{0.254} & 0.774\\ \hline\hline
\multicolumn{1}{l|}{\multirow{4}{*}{\rotatebox[origin=c]{90}{\textbf{OID-PPO}}}}&
 Square    & \multicolumn{1}{c|}{3.181} & \multicolumn{1}{c|}{\textbf{0.009}} & \multicolumn{1}{c|}{\textbf{0.026}} & \textbf{0.971} &
              \multicolumn{1}{c|}{4.775} & \multicolumn{1}{c|}{\textbf{0.014}} & \multicolumn{1}{c|}{\textbf{0.041}} & \textbf{0.945} &
              \multicolumn{1}{c|}{6.292} & \multicolumn{1}{c|}{\textbf{0.023}} & \multicolumn{1}{c|}{\textbf{0.063}} & \textbf{0.938}\\ \cline{2-14}
\multicolumn{1}{l|}{}&
 Rectangle & \multicolumn{1}{c|}{3.321} & \multicolumn{1}{c|}{\textbf{0.012}} & \multicolumn{1}{c|}{\textbf{0.032}} & \textbf{0.962} &
              \multicolumn{1}{c|}{5.011} & \multicolumn{1}{c|}{\textbf{0.015}} & \multicolumn{1}{c|}{\textbf{0.044}} & \textbf{0.941} &
              \multicolumn{1}{c|}{6.439} & \multicolumn{1}{c|}{\textbf{0.020}} & \multicolumn{1}{c|}{\textbf{0.052}} & \textbf{0.932}\\ \cline{2-14}
\multicolumn{1}{l|}{}&
 L-Shape   & \multicolumn{1}{c|}{3.496} & \multicolumn{1}{c|}{\textbf{0.028}} & \multicolumn{1}{c|}{\textbf{0.077}} & \textbf{0.904} &
              \multicolumn{1}{c|}{5.295} & \multicolumn{1}{c|}{\textbf{0.030}} & \multicolumn{1}{c|}{\textbf{0.093}} & \textbf{0.852} &
              \multicolumn{1}{c|}{6.617} & \multicolumn{1}{c|}{\textbf{0.039}} & \multicolumn{1}{c|}{\textbf{0.103}} & \textbf{0.803}\\ \cline{2-14}
\multicolumn{1}{l|}{}&
 U-Shape   & \multicolumn{1}{c|}{3.642} & \multicolumn{1}{c|}{\textbf{0.028}} & \multicolumn{1}{c|}{\textbf{0.079}} & \textbf{0.893} &
              \multicolumn{1}{c|}{5.463} & \multicolumn{1}{c|}{\textbf{0.043}} & \multicolumn{1}{c|}{\textbf{0.117}} & \textbf{0.801} &
              \multicolumn{1}{c|}{6.781} & \multicolumn{1}{c|}{\textbf{0.061}} & \multicolumn{1}{c|}{\textbf{0.124}} & \textbf{0.746}\\ \hline\hline\hline
\end{tabular}
\end{table*}

\begin{table}[h]
\centering
\caption{Performance Metrics for $F_n = 6$ (Ablation Study)}
\label{table:ablation}
\renewcommand{\arraystretch}{1.2}\setlength{\tabcolsep}{1mm}\small
\begin{tabular}{cc||cccc}
\hline\hline\hline
\multicolumn{2}{l||}{\multirow{2}{*}{}} & \multicolumn{4}{c}{$F_n = 6$}\\ \cline{3-6}
\multicolumn{2}{l||}{} & \multicolumn{1}{c|}{Time (s)} & \multicolumn{1}{c|}{P-Loss} & \multicolumn{1}{c|}{V-Loss} & Reward\\ \hline\hline

\multicolumn{1}{l|}{\multirow{4}{*}{\rotatebox[origin=c]{90}{\textbf{OID-PPO}}}}&
Square     & \multicolumn{1}{c|}{4.775} & \multicolumn{1}{c|}{0.014} & \multicolumn{1}{c|}{0.041} & 0.945\\ \cline{2-6}
\multicolumn{1}{l|}{}&
Rectangle  & \multicolumn{1}{c|}{5.011} & \multicolumn{1}{c|}{0.015} & \multicolumn{1}{c|}{0.044} & 0.941\\ \cline{2-6}
\multicolumn{1}{l|}{}&
L-Shape    & \multicolumn{1}{c|}{5.295} & \multicolumn{1}{c|}{0.030} & \multicolumn{1}{c|}{0.093} & 0.852\\ \cline{2-6}
\multicolumn{1}{l|}{}&
U-Shape    & \multicolumn{1}{c|}{5.463} & \multicolumn{1}{c|}{0.043} & \multicolumn{1}{c|}{0.117} & 0.801\\ \hline\hline

\multicolumn{1}{l|}{\multirow{4}{*}{\rotatebox[origin=c]{90}{\textbf{OID-ASC}}}}&
Square     & \multicolumn{1}{c|}{5.632} & \multicolumn{1}{c|}{0.046} & \multicolumn{1}{c|}{0.123} & 0.867\\ \cline{2-6}
\multicolumn{1}{l|}{}&
Rectangle  & \multicolumn{1}{c|}{5.902} & \multicolumn{1}{c|}{0.048} & \multicolumn{1}{c|}{0.128} & 0.859\\ \cline{2-6}
\multicolumn{1}{l|}{}&
L-Shape    & \multicolumn{1}{c|}{6.274} & \multicolumn{1}{c|}{0.061} & \multicolumn{1}{c|}{0.163} & 0.789\\ \cline{2-6}
\multicolumn{1}{l|}{}&
U-Shape    & \multicolumn{1}{c|}{6.484} & \multicolumn{1}{c|}{0.072} & \multicolumn{1}{c|}{0.188} & 0.742\\ \hline\hline

\multicolumn{1}{l|}{\multirow{4}{*}{\rotatebox[origin=c]{90}{\textbf{OID-NIL}}}}&
Square     & \multicolumn{1}{c|}{3.524} & \multicolumn{1}{c|}{0.089} & \multicolumn{1}{c|}{0.222} & 0.512\\ \cline{2-6}
\multicolumn{1}{l|}{}&
Rectangle  & \multicolumn{1}{c|}{3.673} & \multicolumn{1}{c|}{0.094} & \multicolumn{1}{c|}{0.233} & 0.498\\ \cline{2-6}
\multicolumn{1}{l|}{}&
L-Shape    & \multicolumn{1}{c|}{3.978} & \multicolumn{1}{c|}{0.107} & \multicolumn{1}{c|}{0.279} & 0.446\\ \cline{2-6}
\multicolumn{1}{l|}{}&
U-Shape    & \multicolumn{1}{c|}{4.106} & \multicolumn{1}{c|}{0.118} & \multicolumn{1}{c|}{0.305} & 0.418\\ \hline\hline\hline
\end{tabular}
\end{table}

\subsection{Quantitative Study}
\label{subsec:estimation_uni}
We conduct a quantitative evaluation comparing various models on the OID task. The methods include two optimization-based algorithms, including Metropolis-Hastings (MH)~\cite{chib1995understanding} and Multi-Objective Particle Swarm Optimization (MOPSO)~\cite{mopso}, and four deep reinforcement learning (DRL) models with continuous action spaces, including Deep Deterministic Policy Gradient (DDPG)~\cite{ddpg}, Twin-Delayed DDPG (TD3)~\cite{td3}, Soft Actor-Critic (SAC)~\cite{sac}, and our proposed OID-PPO. All DRL models share a common environment interface to ensure a fair comparison.

Table~\ref{table:oidppo} presents performance metrics across various room shapes and furniture quantities. Since the OID-PPO agent observes only the current layout and two furniture items at each step, the environment is partially observable. OID-PPO consistently achieves the highest total reward, demonstrating effectiveness in generating high-quality interior layouts that satisfy both functional and visual criteria. Among the baseline DRL models, SAC outperforms TD3 and DDPG, highlighting the advantage of diagonal Gaussian policies in partially observable environments with complex spatial constraints.

SAC and OID-PPO, which utilize diagonal Gaussian policies, achieve higher layout quality than deterministic agents like TD3 and DDPG. The stochastic nature of these policies enables sampling from a distribution, facilitating broader exploration and better inference of unobserved environmental dynamics. In contrast, deterministic policies select the same action for a given state, often resulting in suboptimal decisions in ambiguous or complex spaces. Non-learning-based methods such as MH and MOPSO perform poorly in both reward attainment and computational efficiency.

OID-PPO updates its policy directly from new samples, following an on-policy paradigm that yields stable convergence and low losses in both policy and value learning. While SAC is off-policy, its entropy-augmented updates help maintain moderate loss values. In contrast, TD3 and DDPG exhibit high variance and unstable convergence due to limited stochastic exploration. Optimization-based methods, such as MH and MOPSO, do not learn value functions and are therefore excluded from loss comparisons.

In terms of inference time, DDPG and TD3 are the fastest due to their lightweight architectures and deterministic action selection. SAC and OID-PPO require more time because of sampling and more complex policy updates, yet they remain efficient for real-world deployment. In contrast, MH and MOPSO are significantly slower due to iterative search procedures and the absence of learned priors, lacking any learning-based acceleration.

As the number of furniture items $F_n$ increases or the room shape becomes more complex (e.g., U- or L-shape), all methods show reduced performance. The decline is most pronounced for MH and MOPSO, which depend on iterative sampling without learned priors, resulting in longer runtimes and lower layout quality. Among DRL models, deterministic agents such as DDPG and TD3 struggle in high-dimensional constrained spaces due to their limited exploration capabilities. In contrast, stochastic models such as SAC and OID-PPO remain more robust, leveraging action diversity and adaptability to partial observability to sustain higher rewards and reasonable inference times despite growing complexity.

The quantitative study shows that DRL methods outperform optimization-based approaches on the OID task. Among DRL models, those with diagonal Gaussian policies, such as OID-PPO and SAC, consistently outperform deterministic-policy agents under high spatial constraints and partial observability. OID-PPO, in particular, benefits from on-policy learning and structured exploration, achieving the best dynamics and layout quality across all conditions. These results establish OID-PPO as the most robust and effective framework for optimal interior design across varied room shapes and furniture complexities.

\subsection{Ablation Study}
\label{sec:ablation}
To assess the contribution of each guideline, we conduct an ablation study by selectively disabling individual reward functions from the aggregate reward. For controlled evaluation, we fix the room shape to a square and set the furniture count to $F_n=6$, balancing layout complexity and feasibility. Functional pairs, such as desk–chair and bed–side table, are predefined.

\begin{figure*}[t]
    \centering
    \includegraphics[width=\textwidth]{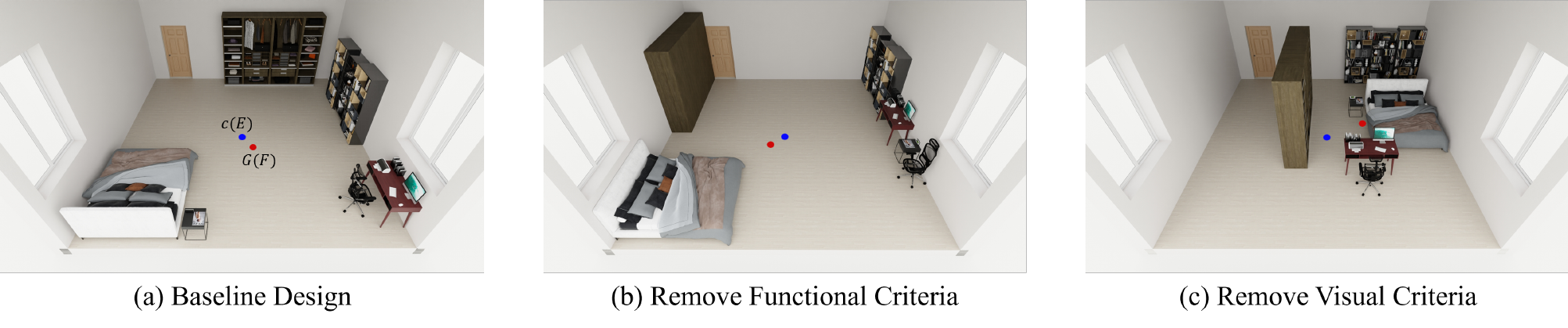}
    \caption{Ablation study under different reward configurations. Blue and red dots indicate the room center $\mathbf{o}$ and the furniture centroid $\bar{\mathbf{x}}_F$, respectively. (a) shows the baseline with all six rewards enabled. (b) shows the effect of removing all functional rewards, causing blocked entrances and reduced usability. (c) depicts the impact of disabling visual constraints, resulting in a cluttered, spatially unbalanced layout. Further ablation results are in Appendix~A.}
    \label{fig:ablation_main}
\end{figure*}

Figure~\ref{fig:ablation_main}~(a) shows the baseline layout with all six reward functions enabled. The blue dot marks the room's geometric center $\mathbf{o}$, and the red dot marks the layout's spatial center $\bar{\mathbf{x}}_F$. All constraints are met: pairwise items are placed closely and aligned ($R_{\text{pair}}$), no obstructions exist near furniture ($R_a$, $R_v$), valid paths connect all items to the door ($R_{\text{path}}$), and the layout is balanced and aligned with room boundaries ($R_b$, $R_{\text{al}}$). Figure~\ref{fig:ablation_main}~(b) shows the effect of removing all functional rewards: the entrance is blocked, paired items are scattered, and usability is compromised despite visual balance. In contrast, Figure~\ref{fig:ablation_main}~(c) turns off all visual constraints; the layout remains functionally valid but appears cluttered and spatially imbalanced, with $\bar{\mathbf{x}}_F$ deviating markedly from the room center $\mathbf{o}$. Additional results—including the impact of removing individual reward terms and modifying placement order or spatial encoding—are presented in Appendix~A. These ablations highlight the distinct and complementary roles of each OID-PPO component in producing high-quality layouts.

Overall, the ablation study shows that individual reward terms encode distinct, complementary design principles. The architectural design of OID-PPO — combining spatial encoding, prioritized placement, and structured rewards — is critical for optimal interior layout generation.
\section{Discussion}

Our results demonstrate that OID-PPO effectively translates expert design knowledge into a computationally tractable framework for interior layout generation. Through structured reward design and stochastic policy learning, it achieves strong performance across diverse room configurations while ensuring stability and adaptability. These findings establish OID-PPO as a practical solution for automating interior design under complex spatial and functional constraints.

Despite the effectiveness of OID-PPO, direct comparison with existing interior design systems remains challenging. Prior DL approaches often rely on proprietary datasets and style-specific goals, which limit the assessment of generalization~\cite{csenbacslar2021rlss}. Earlier RL methods used custom environments with inconsistent rewards and action spaces, lacking standardized protocols for reproducibility~\cite{wang2019rlayout, ieq, csenbacslar2021rlss}. This highlights the need for a shared benchmark that integrates both functional and visual guideline compliance—providing a common platform for rigorous evaluation, fair comparison, and accelerated progress in automated interior design.

A key limitation of OID-PPO is the lack of user preference modeling and real-world constraints. The current framework limits the furniture set to 15 items of a single style, disregarding user aesthetics, color palettes, and factors such as natural or multi-source lighting. It also assumes axis-aligned walls, restricts rotations to $90^\circ$ increments, and excludes vertical placement. Addressing these requires extending the reward function to capture user preferences and illumination, and enabling six-degree-of-freedom, multi-agent actions for curved boundaries, acute angles, and 3D reasoning. These improvements are vital for making OID-PPO a robust backbone for personalized interior design.
\section{Conclusion}

This work presents OID-PPO, a reinforcement learning framework that encodes expert interior design knowledge into structured functional and visual rewards. By embedding guideline compliance directly into the learning objective and using a diagonal Gaussian policy with continuous actions, OID-PPO effectively generates high-quality layouts under tight spatial constraints and partial observability. Experiments conducted across diverse room shapes and furniture counts demonstrate its superiority over optimization-based solvers and DRL baselines in terms of reward and layout quality. Ablation studies validate the necessity and complementarity of the six reward components, underscoring the framework’s robustness. Although currently limited to axis-aligned walls, quantized rotations, and fixed styles, OID-PPO lays a foundation for future extensions incorporating user preferences, lighting, and multi-agent reasoning. This framework provides a robust and extensible foundation for personalized, guideline-compliant interior design automation.

\bibliography{aaai2026}

\newpage \onecolumn \appendix

\section{Appendix A. Ablation Study Results}
\label{appendix:ablation}

\begin{figure*}[t]
    \centering
    \includegraphics[width=0.8\textwidth]{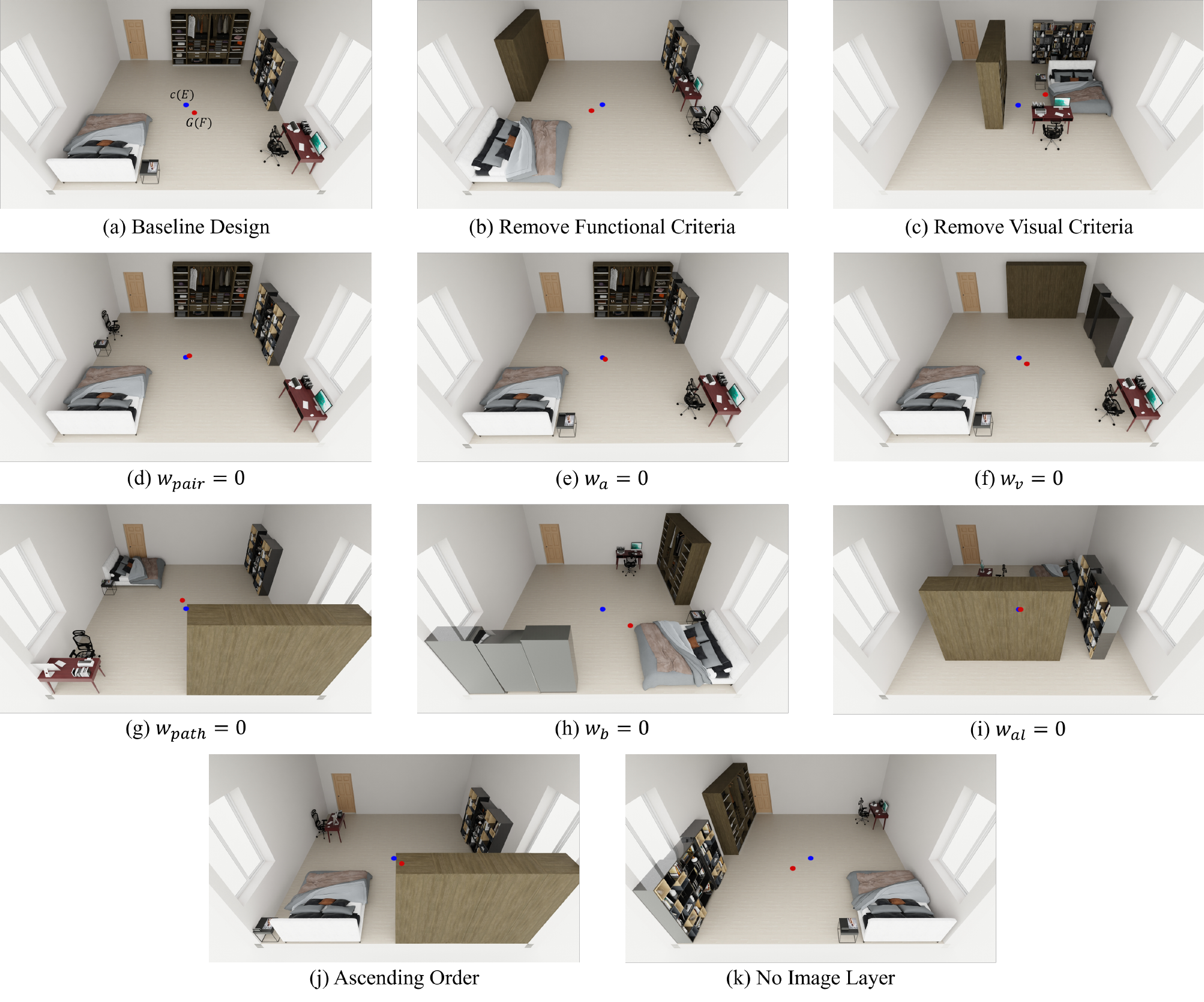}
    \caption{Extended ablation results under various reward and architecture settings. Blue and red dots mark the room center $\mathbf{o}$ and layout centroid $\bar{\mathbf{x}}_F$, respectively. (a--c) replicate key ablations from the main study. (d--g) show effects of removing individual functional rewards: pairwise relations, accessibility, visibility, and pathway connectivity. (h--i) isolate visual components: balance and alignment. (j--k) assess architectural variations: item ordering and spatial encoding. Each case highlights the distinct, complementary roles of these components in layout quality and agent behavior.}
    \label{fig:ablation}
\end{figure*}

To extend the analysis, we examine the effects of removing individual reward components and architectural changes. For consistency, we fix the room shape as a square with $F_n = 6$ furniture items. Functional pairs (e.g., desk–chair, bed–side table) are predefined to reflect typical use.

Figure~\ref{fig:ablation}~(a--c) present the representative cases from the ablation study section and are included here for completeness. Figures~\ref{fig:ablation} (d--g) illustrate the effects of disabling individual functional rewards. In (d), removing $R_{\text{pair}}$ improves balance, but the desk and chair are separated, breaking their functional pairing. In (e), disabling $R_a$ slightly improves visual balance, but the wardrobe occludes the bookshelf, reducing accessibility. In (f), disabling $R_v$ causes furniture to face walls, rendering them unusable. In (g), removing $R_{\text{path}}$ leads to a blocked entrance despite individual item accessibility, making the layout infeasible.

Figures~\ref{fig:ablation} (h--i) illustrate the effects of removing visual rewards. In (h), removing $R_b$ disrupts visual balance, shifting the spatial center away from the geometric center of the room, although functional constraints remain satisfied. Since $R_{\text{al}}$ is active, most of the furniture remains aligned along the walls. In (i), without $R_{\text{al}}$, items are inefficiently placed away from the walls, increasing the wasted space, but the overall spatial distribution remains balanced.

Figure~\ref{fig:ablation} (j) shows the result of placing furniture in ascending size order. Small items are placed in corners first, violating alignment constraints and reducing spatial efficiency. As shown in Table~\ref{table:ablation}, \textit{OID-ASC} leads to higher policy and value losses, with increased inference time, indicating that the placement of larger furniture first yields more efficient layouts in OID-PPO. In contrast, (k) is \textit{OID-NIL} shows substantially lower rewards and elevated losses, suggesting that the removal of spatial encoding impairs the agent's ability to interpret room geometry and meet spatial constraints. Together, these findings underscore the crucial role of each OID-PPO component in producing well-structured and high-quality interior layouts.

\section{Appendix B. Proof of Theorem 1}
\label{appendix:proof}

\begin{proof}[Proof]
    We prove convergence of the policy sequence $\{\pi_{\theta_k}\}_{k\ge 0}$ under the conditions of Theorem~1, following the standard stochastic approximation framework and the ODE method. Each step corresponds to one of the key assumptions and intermediate results stated in the supporting propositions and lemmas.

    \textbf{Step 1 (Bounded Return).} By Proposition~\ref{prop:horizon} and Lemma~\ref{lemma:reward}, each episode terminates after finite steps $H=|F|$, and rewards are bounded within $[-1,1]$, thus cumulative returns are bounded:
    \begin{equation*}
        \left|\sum_{t=0}^{H}\gamma^t R_{idg}(s_t,a_t)\right|\le\frac{1-\gamma^{H+1}}{1-\gamma}\le\frac{1}{1-\gamma}=C<\infty.
    \end{equation*}
    \textbf{Step 2 (Stable Gradient Estimates).} Since the PPO algorithm and GAE provide bounded and low-variance advantage estimates, the gradients remain stable and bounded:
    \begin{equation*}
        \|\nabla_\theta L_{\text{clip}}(\theta)\|\le L_g<\infty,\quad\forall \theta\in\Theta.
    \end{equation*}
    \textbf{Step 3 (Robbins–Monro Condition and ODE Approximation).} The Adam optimizer with decaying learning rate effectively satisfies the Robbins–Monro conditions:
    \begin{equation*}
        \sum_{k=0}^{\infty}\alpha_k=\infty,\quad \sum_{k=0}^{\infty}\alpha_k^2<\infty.
    \end{equation*}
    Under these conditions, the parameter updates approximate the ODE:
    \begin{equation*}
        \dot{\theta}(t)=\nabla_\theta L_{\text{clip}}(\theta(t)).
    \end{equation*}
    By Borkar–Meyn stochastic approximation theory, we obtain almost sure convergence:
    \begin{equation*}
        \theta_k\xrightarrow{\text{a.s.}}\Theta_\infty\subseteq\{\theta:\nabla_\theta L_{\text{clip}}(\theta)=0\}.
    \end{equation*}
    \textbf{Step 4 (Lyapunov Stability).} Assume that Lyapunov function is defined as $V(\theta)=-L_{\text{clip}}(\theta)$. Its time derivative satisfies:
    \begin{equation*}
        \langle\nabla V(\theta),\dot{\theta}\rangle=-\|\nabla L_{\text{clip}}(\theta)\|^2\le0,
    \end{equation*}
    thus ensuring stability and convergence to a stationary point.

    \textbf{Step 5 (Surrogate and True Return Approximation).} Proposition~\ref{prop:monotonic} guarantees bounded bias between surrogate and true returns:
    \begin{equation*}
        |J(\theta)-L_{\text{clip}}(\theta)|\le O(\epsilon),\quad\epsilon\ll1.
    \end{equation*}
    As $\theta_k\rightarrow\theta^*$, the continuity of $L_{\text{clip}}$ and $J$ implies:
    \begin{equation*}
        |J(\theta_k)-J(\theta^*)|\rightarrow 0,\quad\text{a.s.}
    \end{equation*}
    Thus, the return converges almost surely:
    \begin{equation*}
        J(\theta_k)\xrightarrow{\text{a.s.}}J(\theta^*).
    \end{equation*}
    \textbf{Step 6 (Persistent Exploration and Monotonic Improvement).} By Proposition~\ref{prop:exploration}, exploration is ensured by the positive variance of the diagonal Gaussian policy, guaranteeing infinite visits to the feasible action space. By Proposition~\ref{prop:monotonic}, monotonic improvement in returns (up to \( O(\epsilon) \)) is guaranteed during training:
    \begin{equation*}
       J(\theta_{k+1})\ge J(\theta_k)-O(\epsilon). 
    \end{equation*}
    Combining Steps 1--6, the parameter sequence $\theta_k$ converges almost surely to a stationary point $\theta^* \in \Theta_\infty$ of the clipped surrogate objective. Moreover, the true expected return $J(\theta_k)$ converges to $J(\theta^*)$ within an $O(\epsilon)$ bias, with $J(\theta_k)$ monotonically non-decreasing up to this bias. This establishes the convergence of OID-PPO.
\end{proof}

\end{document}